\documentclass{article}
\usepackage{ijcai20}

\usepackage{times}
\usepackage{soul}
\usepackage{url}
\usepackage[utf8]{inputenc}
\usepackage[small]{caption}
\usepackage{graphicx}
\usepackage{amsmath}
\usepackage{amsthm}
\usepackage{booktabs}
\usepackage{bbold}
\usepackage{graphicx}
\usepackage{amsmath}
\usepackage{bbm}
\usepackage{mathtools}
\usepackage{url}
\usepackage{enumitem}

\usepackage{graphicx}

\usepackage{times}
\usepackage{soul}
\usepackage{url}
\usepackage[utf8]{inputenc}
\usepackage{amsmath}
\usepackage{booktabs}
\usepackage{algorithm}
\usepackage{algorithmic}

\usepackage{pbox}
\usepackage{subcaption}
\usepackage{multirow}
\usepackage{mathrsfs} 
\usepackage{bbm}
\usepackage{colortbl}
\usepackage{commath}
\usepackage{centernot}
\usepackage{amssymb}
\usepackage{mathtools}
\usepackage{afterpage}

\newtheorem{thm}{Theorem}

\newtheorem{proposition}{Proposition}

\renewcommand{\epsilon}{\varepsilon}

\title{Infochain: A Decentralized, Trustless and Transparent Oracle on Blockchain}

\author{
	Naman Goel$^1$\thanks{The two authors contributed equally.}
	\and
	Cyril van Schreven$^1$\footnotemark[1]\and
	Aris Filos-Ratsikas$^2$\And
	Boi Faltings$^{1}$
	\affiliations
	$^1$Swiss Federal Institute of Technology, Lausanne (EPFL)\\
	$^2$University of Liverpool, UK\\
	\emails
	\{naman.goel, boi.faltings\}@epfl.ch, cyril.schreven@protonmail.com,
	aris.filos-ratsikas@liverpool.ac.uk
}

\begin{document}

\maketitle

\begin{abstract}
Blockchain based systems allow various kinds of financial transactions to be executed in a decentralized manner. However, these systems often rely on a trusted third party (oracle) to get correct information about the real-world events, which trigger the financial transactions. In this paper, we identify two biggest challenges in building decentralized, \textit{trustless}\footnote{Trustless is a term increasingly used in the context of decentralized and blockchain systems meaning \textit{not requiring trust}.} and transparent oracles. The first challenge is acquiring correct information about the real-world events without relying on a trusted information provider. We show how a peer-consistency incentive mechanism can be used to acquire truthful information from an untrusted and self-interested crowd, even when the crowd has outside incentives to provide wrong information. The second is a system design and implementation challenge. For the first time, we show how to implement a trustless and transparent oracle in Ethereum. We discuss various non-trivial issues that arise in implementing peer-consistency mechanisms in Ethereum, suggest several optimizations to reduce gas cost and provide empirical analysis.
\end{abstract}

\section{Introduction}
With the increasing popularity of the blockchain technology, the implementation of commercial and governmental systems has witnessed a large shift towards distributed and decentralized approaches. In particular, the emergence of the Ethereum platform has given rise to the development of several applications, often referred to as \emph{decentralized apps} or \emph{DAPs}, which aim to apply this latter principle to many areas such as finance, education, intellectual property or government. At the heart of these approaches lies the concept of the \emph{smart contract}, i.e., lines of code that contain the terms of the agreement between the involved parties, which are automatically executed once triggered by events happening in the real world. For example, consider the case of a web service, which is typically dictated by a service level agreement (SLA) between the service provider and the clients. The SLA can be coded into a smart contract between the involved parties which will trigger an automatic payment upon detection of a violation. For instance, if the service guarantees a response time of at most $1$ second with high probability, frequent slower responses would trigger automatic compensation. An important issue here is, \emph{how to determine whether the real-world event has actually happened}. In the above example, this means how to determine that the SLA has been violated? We use the case of the web service only as a simple running example but this is in fact a \emph{fundamental challenge in developing information infrastructure for FinTech}. 

The need for trusted information about a real-world event that triggers some conditional financial transactions arises in applications ranging from insurance, banking, trade, governance and law etc. The entities responsible for acquiring such data about the real-world events are called \emph{oracles}. Existing solutions include Town Crier and Chainlink among others. Traditionally, oracles are implemented using trusted third party data sources responsible for acquiring the information. However, besides the fact that such an approach is in conflict with the decentralized nature of the blockchain technology, it is also prone to problems such as trustworthiness and cost.

An alternative solution would be to appeal to the ``wisdom of the crowds'' and ask the users themselves about the information (for e.g., the quality of service received). The idea has also been proposed for outcome resolution in decentralized prediction markets like Augur and Gnosis. While this approach is more decentralized in nature, it poses a significant challenge: the agents can not be relied on to provide correct information. The task of eliciting information from self-interested agents is one of the fundamental problems in game theory, and has been extensively studied. In the center of these investigations lies the literature on peer-consistency mechanisms~\cite{faltings2017game}; these are game-theoretic mechanisms that incentivize agents to report the information truthfully, even if the information is unverifiable. 

In this paper, we propose the employment of peer-consistency mechanisms for the design of trustless, decentralized oracles. This seems like a natural choice, as the usefulness of the oracles is dependent on the quality of the supplied information, which needs to be truthfully elicited from the agents. However, this quest imposes two major challenges:

\begin{itemize}
	\item For a peer-consistency mechanism to actually work, the agents must be convinced of its incentive properties, contrary to the traditional case, where the implementation of the mechanism is done by a trusted third-party. \emph{How can one implement the incentive scheme in a transparent and trustless manner, what is the cost and how can we optimize this cost?}
	
	\item In many financial settings, agents also have incentives to lie about their true observations and provide false information. In the web service example, the clients would have an incentive to always report ``bad'' response times, in order for the conditions of the smart contract to be violated in their favor. \emph{How large do the incentives have to be, to counteract the lying incentives, and is the approach economically feasible?}
\end{itemize}

\noindent In this paper, we address the above questions. We summarise our contributions below.
\begin{enumerate}
		\item We design and implement \textsc{Infochain}, a completely decentralized peer-consistency based truthful information collection system in Ethereum. We address the following technical challenges in its implementation.
	
	\begin{itemize}
		\item Writing data and performing computation on Ethereum's Virtual Machine (EVM) is expensive. Information providers must be compensated for this cost, increasing the overall cost of information acquisition. For the first time, we discuss several non-trivial ways of implementing three different peer-consistency mechanisms in Solidity (Ethereum's programming language) and empirically compare their costs.
		
		\item While transparency is a desired inherent feature of blockchain, the peer-consistency mechanisms are compromised if an agent can see the information submitted by their ``peers" before submitting their own information. We propose to use a \emph{commit-reveal protocol} to address this challenge.
		
		\item In order to reduce computation complexity, peer-consistency mechanisms use only one (or a few) randomly selected peer(s) for every agent. However, if the random peer(s) can be predicted, the agents get an opportunity to collude and the mechanisms can be compromised, and this risk is increased by the transparency of blockchain. We show that under reasonable assumptions, random peer selection can still be implemented safely.
	\end{itemize}

	\item We analyze the settings when agents have outside incentives to lie.  We formally show that even in the presence of such outside incentives peer-consistency can be used to elicit the truth by choosing an appropriate constant to scale the rewards. We show that normally the payments required are a small fraction of the outside incentive.
\end{enumerate}

\subsection{Related Work}
Many decentralized systems have been proposed for crowdsourcing and information trading~\cite{an2019truthful,xiong2019smart,lu2018zebralancer,li2018crowdbc} but none addresses the issue of providing quality based incentives for information. In a recent and independent work, \cite{kong2019securely} also use peer-consistency for \textit{trust-free} data trading systems but the analysis in this theoretical paper focuses on a secure multi-party computation protocol for rewarding information that loses value if revealed. \cite{adler2018astraea} propose a system for a decentralized oracle but it requires a ``random assignment" of questions to agents, which has the drawback that agents may be asked to answer questions that they may have no information about.

The literature on peer-consistency is long and extensive \cite{miller2005eliciting,prelec2004bayesian,waggoner2014output,gao-new,agarwal-2017,liu2017machine,goel2018deep,goel2019personalized}.; we refer the reader to \cite{faltings2017game} for a detailed exposition of the main results. Most relevant to ours are the results related to incentives which are dependent on the outcome, e.g., \cite{chakraborty2016trading,chen2011market,freeman2017crowdsourced}. These differ from our theoretical contributions in the fact that they only apply to the specific domain of prediction markets.

\section{Infochain}
\begin{figure*}[ht]
	\centering
	\includegraphics[width=1\linewidth]{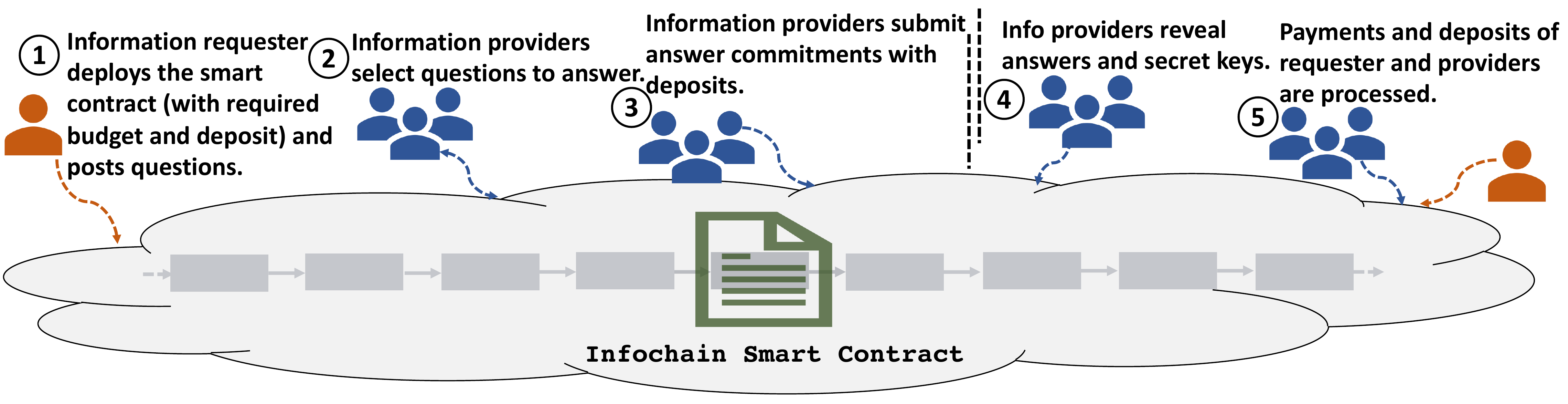}
	\caption{Infochain Overview}
	\label{fig:design}
\end{figure*}
To collect truthful information from self-interested agents, we propose a completely decentralized, transparent and trustless system called Infochain. Infochain enables information requesters to post questions, which can be selected by information providers (agents). The questions can be, for example, of the following form: ``Is the \texttt{responseTime} of web service $W$ less than $1$ second?". Once the agents submit information for the questions they select to answer, their payments in Ether are processed by a smart contract. All the collected information and payments are stored on a public blockchain to ensure transparency and immutability.
\paragraph{Peer-consistency.}A crucial step in eliciting trustworthy information from self-interested agents is aligning their incentives with honest behavior. Unfortunately, a naive incentive mechanism may invite free riders who submit random information. Designing truthful incentive mechanisms is a hard problem when there is no way to verify the correctness of the information. This issue has been addressed by game theoretic peer-consistency mechanisms~\cite{faltings2017game}. The broad idea in these mechanisms is to reward the agents by ``matching" the information provided by multiple agents, while discouraging any collusion. The state of the art peer-consistency mechanisms guarantee that truth-telling strategy is the highest paying equilibrium and other equilibria are less profitable. We consider three fundamental peer-consistency mechanisms in this paper.

\begin{enumerate}
	\item \textbf{The Output Agreement (OA) Mechanism~\cite{waggoner2014output}:} This is perhaps the simplest of all peer-consistency mechanisms. In the OA mechanism, an agent gets a reward of 1 unit only if her answer for a question matches the answer of her peer for the same question. The reward of the agent for a question is the average over the rewards earned by matching with all peers. The final reward of the agent is the average of her rewards from all the questions answered by her.
	
	\item \textbf{The Dasgupta and Ghosh (DG) Mechanism~\cite{dasgupta2013crowdsourced}:} In the DG mechanism, an agent gets a reward of 1 unit if her answer for a question matches the answer of her peer for the same question but also gets a penalty of 1 unit if her answers match the answers of the peer on non-common questions. The DG mechanism requires that two agents, who are peers of one another, must also have some non-common questions that are answered by one of them but not by both. The final reward is calculated by averaging as described in the OA mechanism. The \textbf{Correlated Agreement} mechanism~\cite{shnayder2016informed} is a generalization of the DG mechanism and exhibits similar computations.
	
	\item  \textbf{The Peer Truth Serum for Crowdsourcing(PTSC)~\cite{radanovic2016incentives}:} In PTSC, the reward of an agent $i$ for a question is calculated using the following formula:
	$$
	\begin{dcases*}
	\alpha \cdot \Big(\frac{\mathbb{1}_{y = y'}}{R_i(y)} - 1\Big)& if $R_i(y) \neq 0$\\
	0 & if $R_i(y) = 0$
	\end{dcases*}
	$$
	where $y$ is the answer submitted by the agent and $y'$ is the answer submitted by her peer for the same question. $\alpha$ is a strictly positive scaling constant.
	
	$R_i(y) = \text{num}_i(y)/\sum\limits_{\bar{y} \in \{0,1\}} \text{num}_i(\bar{y}),$
	where $\text{num}_i(y)$ is a function that counts occurrences of $y$ in the answers of all agents (except $i$) across all questions. The final reward is calculated by averaging discussed earlier.
\end{enumerate}

Traditionally, these mechanisms are implemented by a centralized trusted third party. Implementing them in Infochain, which doesn't assume any centralization or trust, is challenging. In the following subsections, we address the main implementation and theoretical challenges. An overview of Infochain is provided in Figure~\ref{fig:design}.  

\subsection{Commit-Reveal Protocol} Transparency is an inherent feature of blockchain. Thus, all the information submitted by an agent is visible to all others. The peer-consistency mechanisms guarantee their incentive compatibility assuming that an agent can only form a belief about what her peers are going to report but doesn't know the actual report of peers. We ensure this in Infochain by making the agents follow a commit-reveal protocol:
\begin{enumerate}
	\item \textbf{Commit:} An agent writes her commitment $keccak256(y, k)$ on the chain, where $y$ is the agent's answer for a given question and $k$ is her secret key.
	\item \textbf{Reveal:} Once all agents who have selected a question, have finished submitting their commitments for the question or the commitment phase expires, they can reveal their respective secret keys and answers. If the commitment of an agent matches her revealed answer, the answer is written on the chain, otherwise it is discarded.
\end{enumerate}

\begin{figure*}[!ht]
	\centering
	\begin{subfigure}{0.24\linewidth}
		\centering
		\includegraphics[width=1\linewidth]{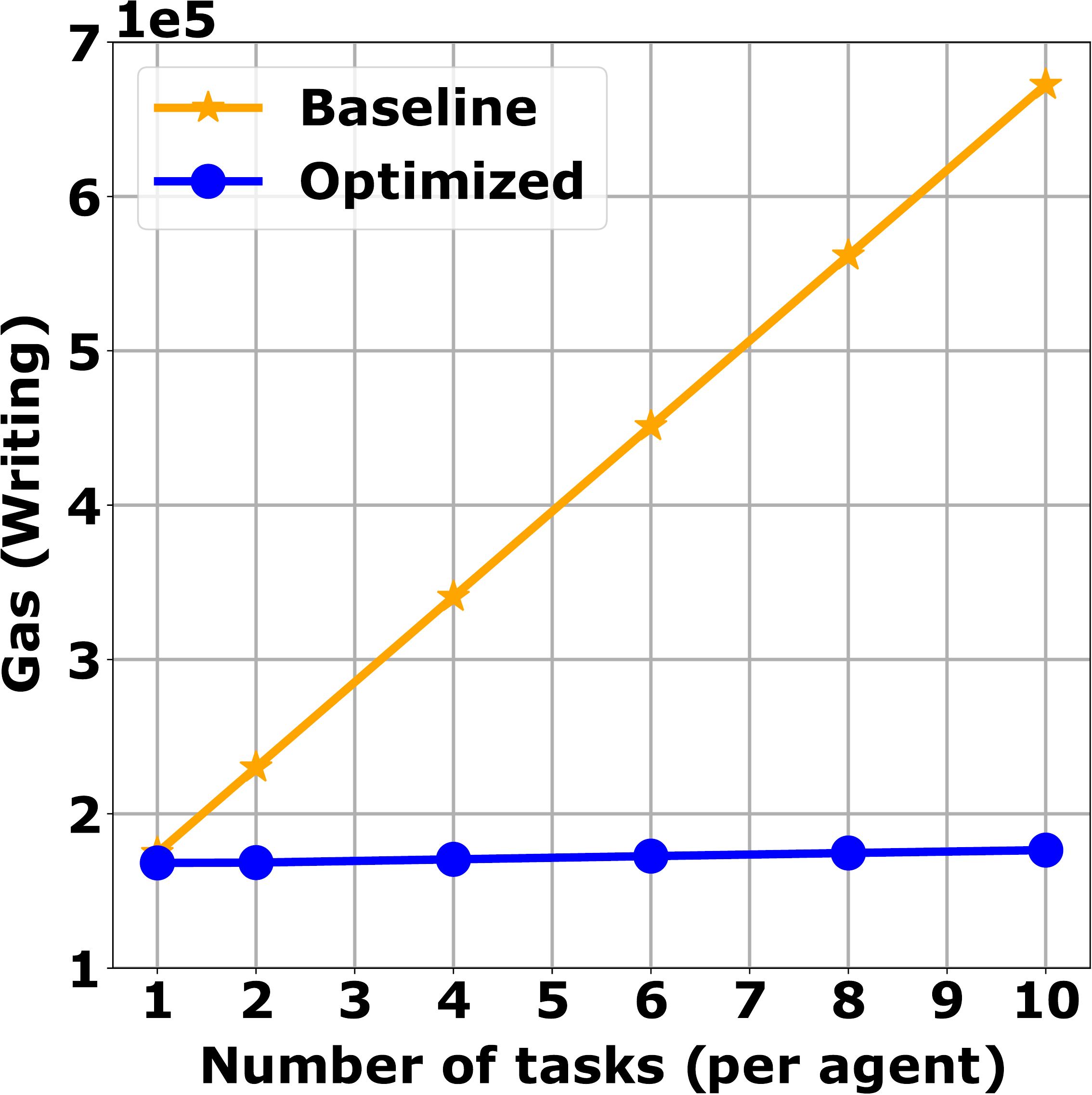}
		\caption{Writing Cost Optimization}
		\label{fig:2_optimize_writing}
	\end{subfigure}~
	\begin{subfigure}{0.24\linewidth}
		\centering
		\includegraphics[width=1\linewidth]{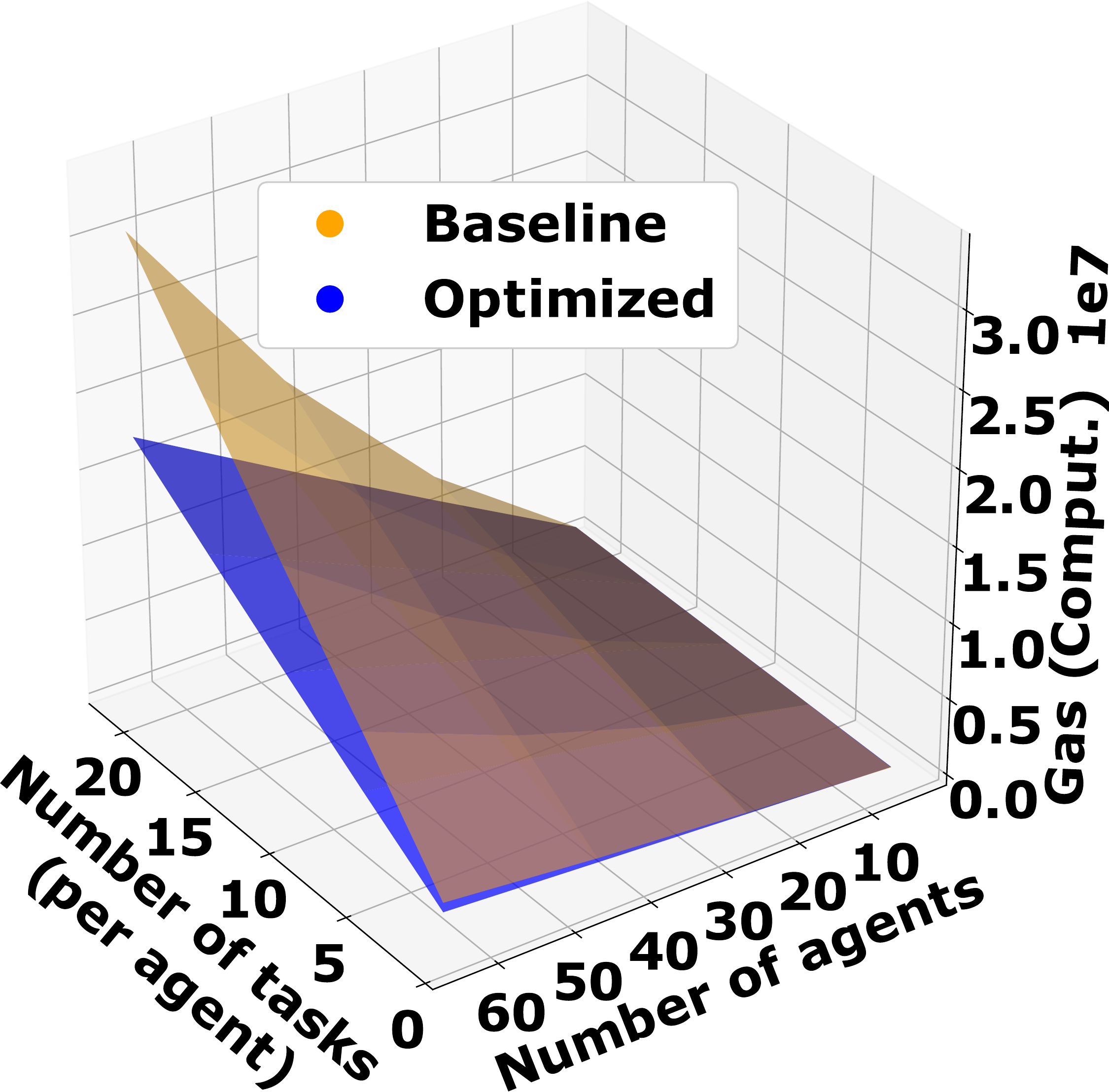}
		\caption{Comput. Cost Optimization}
		\label{fig:1_optimize_computation}
	\end{subfigure}
	\begin{subfigure}{0.24\linewidth}
		\centering
		\includegraphics[width=1\linewidth]{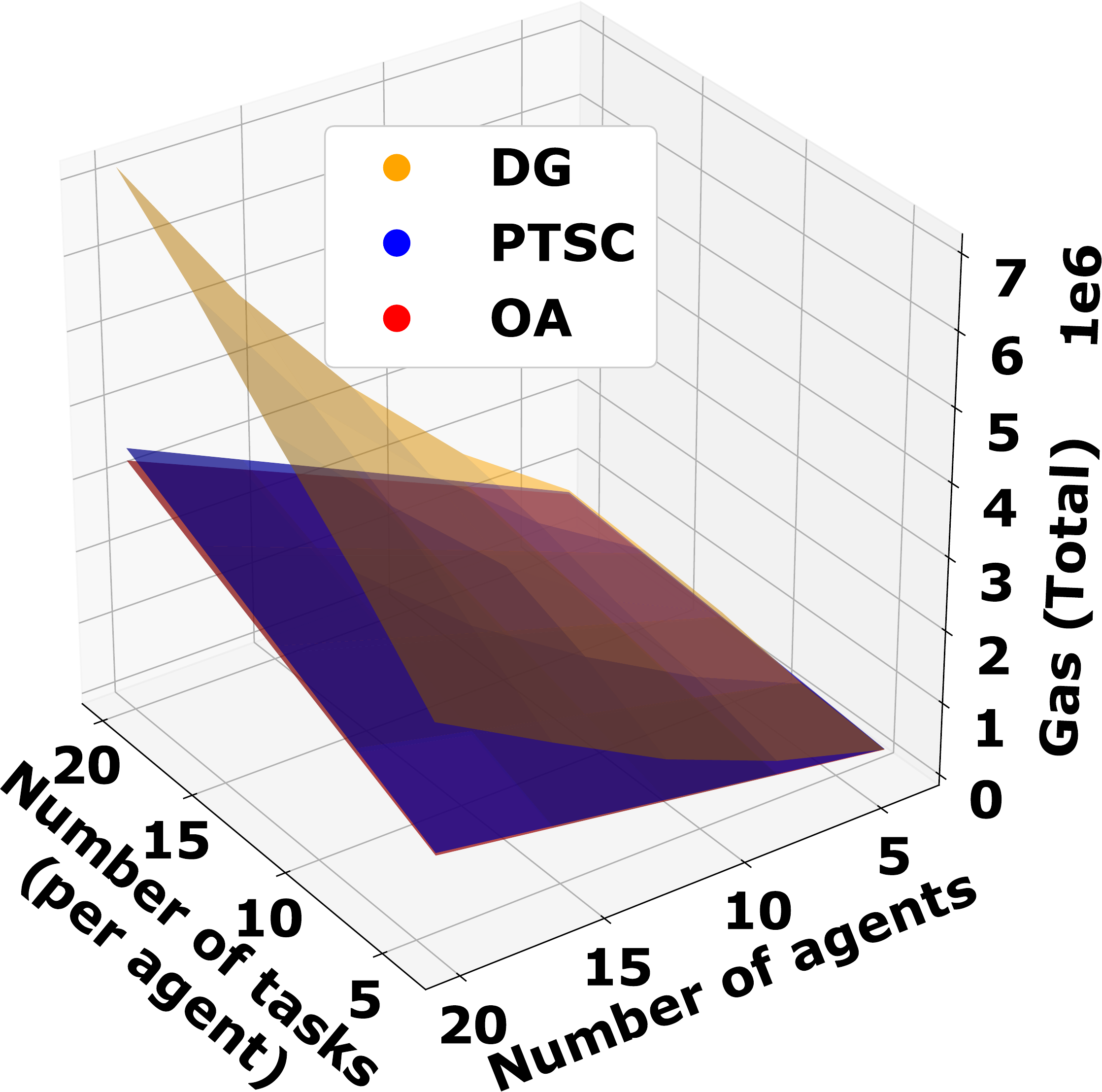}
		\caption{Comparing Mechanisms}
		\label{fig:compare_mech}
	\end{subfigure}
	\begin{subfigure}{0.24\linewidth}
		\centering
		\includegraphics[width=1\linewidth]{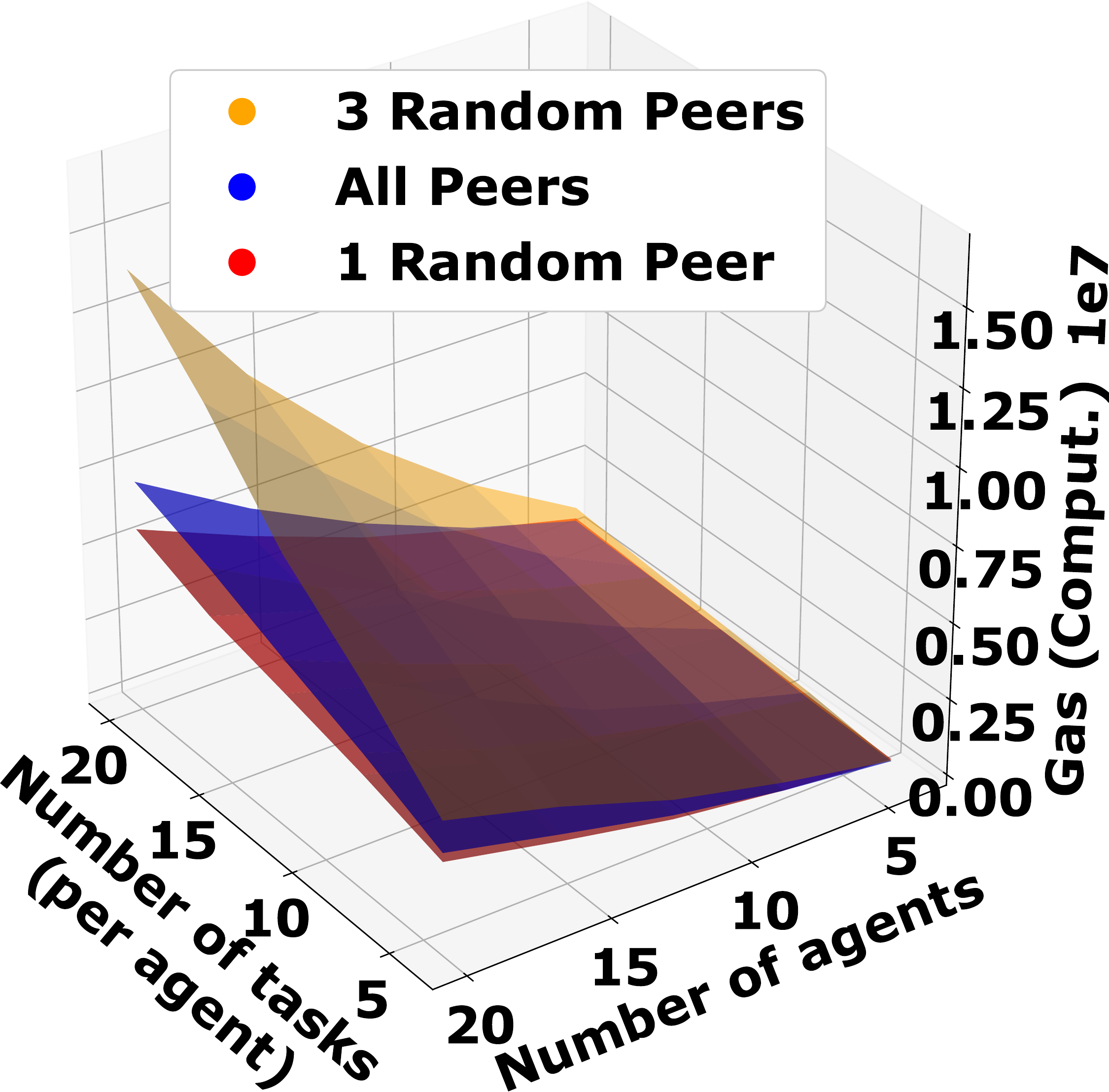}
		\caption{Random Peer Selection (DG)}
		\label{fig:randomness}
	\end{subfigure}
	\caption{Experimental Results}
	\label{fig:results}
\end{figure*}
\subsection{Cost Optimizations}\label{sec:opt}
 Performing computations on Ethereum's Virtual Machine (EVM) remains an expensive affair. Computation costs on EVM are roughly $10^8$ times higher than AWS\footnote{https://aws.amazon.com/blockchain/}. \cite{ryan} provides a good summary about the costs of basic arithmetic operations and writing operations for different data types. Agents who provide information must be compensated for this cost, increasing the overall cost of information acquisition. We discuss below several non-trivial ways of implementing three different peer-consistency mechanisms in Solidity so that the costs can be minimized.

\begin{enumerate}
	\item \textbf{Optimizing Writing Cost}: To minimize the costs of writing on the chain, agents on Infochain combine multiple answers in the form of a bit vector. This is motivated by two observations. First, the answers are revealed simultaneously and thus, they do not require separate commitments. Second, the EVM operates on 256 bit words, thus a single bit vector is much cheaper to write than other formats. 
	\begin{proposition}
	With the above scheme, each 256-bit commitment can contain up to 42 answers.
	\end{proposition}
	\begin{proof}
		Given a hash function $\mathcal{H}$ with a $3k$ bit output, to commit the $k$ bit message $m$, Alice generates a random $k$ bit string $\mathcal{S}$ and sends Bob $\mathcal{H}(\mathcal{S}||m)$. The probability that any $\mathcal{S}'$, $m'$ exist where $m' \neq m$ such that $\mathcal{H}(\mathcal{S}' || m') = \mathcal{H}(\mathcal{S} || m)$ is $\approx 2^{-k}$. The size of the message sent is limited to one third the size of the output of the hashing function, thus $85$ bits.
		Each answer requires $2$ bits: the first determines if the question was answered and the second is the answer. Therefore each commitment can contain 42 answers.
	\end{proof}

This optimization helps both commit and reveal phases.
	
	\item \textbf{Optimizing Computation Cost}: To reduce the cost of computing the rewards, a set of so-called intermediary values is introduced. These values naturally appear at intermediary states of reward computation. They will be precomputed and reused for each agent. What these intermediary values are, depends on the peer-consistency mechanism. This approach allows for the computation to traverse the data a minimum number of times. Since all rewards are computed at the same time, these intermediary values don't need to be written on the blockchain and can be kept in memory instead.
	
	For an example, consider the PTSC mechanism, which requires relative frequency $R_i(y)$ of the value $y$ while excluding the answer given by agent $i$. This quantity need not be calculated from scratch for every agent or when every new answer is submitted and neither it is required to be written on the chain. The intermediary values (for e.g. running average) can be kept in memory and used to calculate or update  $R_i(y)$ as required. 
\end{enumerate}

\subsection{Random Peer Selection}
In peer-consistency, we can use only one or a few randomly selected peers for reward calculation instead of all peers. This is because, in expectation, the rewards of the agents remain unchanged and thus, the mechanisms with randomly selected peers also offer the same incentive compatibility (except that the variance in rewards increases). This is an interesting tradeoff between computation cost and variance in rewards. However, random peer selection on blockchain is subtle mainly due to the fact that nothing on the chain is a ``secret", including the seed for random number generation. If random peers can be known in advance, it may increase the risk of collusion between the agents compromising the incentive compatibility of the mechanisms. In Infochain, we use the block timestamps as well as the mining difficulty level as the seed. This avoids using any trusted third party for random peer selection. The approach works under the assumption that the miners will not try to cheat the smart-contract, which is a reasonable assumption given that the miners have no incentive to do so since they risk losing their mining rewards. The assumption can be violated in extreme scenarios where the financial activity on Infochain (for e.g. the incentive amounts) exceed the mining rewards.

\subsection{Negative Payments}
The DG mechanism and the PTSC allow negative payments, which is implemented in Infochain by making agents submit refundable deposits. Information requesters also deposit the payment budget and an additional refundable deposit. Any outstanding deposits of the agents and the requester are returned after the payments and computations costs are settled.

\section{Experiments}
We now discuss the results of some experiments performed on Infochain. The performance measure of interest in this discussion will be the total amount of gas used. Gas is a unit measuring the computational work of running transactions or smart contracts in the Ethereum network and is a good proxy for the cost in USD. Infochain has been deployed and tested on the Ropsten Test Network, one of the commonly used public testing framework for Ethereum smart contracts. To have no limitations in terms of gas, the results reported in this paper have been generated on a local instance of Ethereum.

\paragraph{Dataset Description.}For this experiment, we used a public dataset~\cite{zheng2014investigating} containing real-world quality of service evaluation from 339 trusted agents for 5,825 web services. The agents observe the response time (in seconds) of the web-services. The real valued observations were placed into two categories (``good'' and ``bad''), in order to fit them to our binary observation setting. We treated a response time of at most 1 second as ``good'' and the rest as ``bad''. This dataset acts as the ground truth data that the information requester is interested in eliciting from self-interested agents. We simulated agent behavior as follows: $50\%$ of the agents report truthfully, $25\%$ report randomly (i.e. independent of the ground truth) and the rest report in an adversarial way (i.e. opposite of the truth).

\paragraph{Results.}In Figure~\ref{fig:2_optimize_writing}, we show the reduction in writing cost due to the proposed optimization discussed in Section~\ref{sec:opt} as compared to the baseline implementation (without any proposed optimizations). Tasks in the figures refer to the questions that the agents answer. As expected, the reduction becomes more significant as agents answer more questions since the optimization can pack more and more answers into a single write operation. It may be worth noting that the optimization doesn't make writing cost independent of the number of answers as the figure may suggest. Since the number of questions in the figure doesn't exceed $42$, the cost remains same as number of questions increase. Figure~\ref{fig:1_optimize_computation} shows the reduction in computation cost due to the proposed optimizations with varying number of agents and number of questions per agent. The figure was plotted based on the numbers obtained with the PTSC mechanism but we observed a similar trend for the OA and the DG mechanisms. We next compare the cost of the three mechanisms in Figure~\ref{fig:compare_mech}. While the OA mechanism and the PTSC mechanism have similar cost, the DG mechanism is more costly. This is due to the fact that DG mechanism involves more operations, particularly for keeping track of questions that are not shared between agents. Finally, Figure~\ref{fig:randomness} shows the effect of using randomly selected peers for reward computation in the DG mechanism. We note that there may be multiple ways to implement sampling without replacement; for e.g., 1) randomly select a peer, check if it is already in the list of previously selected peers and repeat; and 2) sample from the list of not selected peers, update the list of not selected peers and repeat. The first method is not suitable for blockchains as there is no upper bound on the number of necessary random selections and thus the transaction may run out of gas. The results presented here correspond to the second method. As shown in Figure~\ref{fig:randomness}, the cost is guaranteed to reduce if we randomly select only one peer per agent. But when multiple peers are to be selected (which is required to reduce variance in rewards), the cost may increase to a level higher than the cost of using all peers without any random selection. The reason for this is that as we select more random peers, the cost of implementing random sampling exceeds the cost of simple implementation of just using all the peers.

\section{Countering Lying Incentives}
In the previous section, we explained how some popular peer-consistency mechanisms can be implemented in Ethereum, which is an essential step towards the design of decentralized oracles. The other crucial part is to ensure that they provide the incentives to the agents to report their observations truthfully, \textit{even if their inherent incentives are towards the other direction}. To make this more concrete, consider the example of the the web service, mentioned in the Introduction. Clearly, the agents have an incentive to report ``bad service'' so that they can be compensated by the service provider. We will show how one of the peer-consistency mechanisms of Infochain, the PTSC mechanism can actually be tuned to counter these \emph{outside} incentives. We remark that while the mechanism is known to be able to handle \emph{constant} extra incentives \cite{radanovic2016incentives}, like the cost of effort, it was not known whether it can be used for incentives that \emph{depend on the outcome and the reports of the other agents}. We establish such a result in this section, and we quantify the savings that the employment of the mechanism achieves, compared to the case of not applying any peer-consistency.

Formally, we consider settings in which there is a large number of questions to be answered, and each agent selects and answers a few of them. The answer space is defined by a binary variable, e.g., ``good service'' or ``bad service''. We will use $x_i \in \{0,1\}$ to denote the (private) observation of agent $i$ and $y_i \in \{0,1\}$ to denote its report. Since agents are rational, it might not be the case that $y_i = x_i$ but rather, $y_i$ will be some function of $x_i$. If $y_i = x_i$ we will say that the agent is being \emph{truthful}. Another important case is when $y_i = 0$ (regardless of $x_i$), where $0$ denotes ``bad service''. The \emph{outcome} $o_q$ for a question $q$ is defined as the fraction of the $n$ agents who reported $0$ as their feedback on that question. Based on the announced outcome, the agents (who submitted $0$ as feedback) are issued a \emph{refund payment} $c \cdot o_q$, i.e., proportional to the value of the outcome. 

Intuitively, the outcome is determined by the agents that claimed to be dissatisfied with the service and asked for a refund. In the web service example, this corresponds to the fraction of agents who report that the response time of the service was higher than the guarantee. It should be obvious that if we do not provide any extra incentives (i.e., in the absence of peer-consistency), every rational agent would report $y_i = 0$, in order to get compensated. We will prove that with the appropriate use of PTSC, one can make sure that being truthful is the (only) best option for an agent, assuming that other agents are also truthful. In game-theoretic terms, we will prove that being truthful is a \emph{strict equilibrium}.\footnote{Since the agents only have subjective beliefs about the observations of others, the appropriate equilibrium concept here is the \emph{subjective equilibrium~\cite{witkowski2012peer}}.} This will be achieved via an appropriate choice of the scaling constant $\alpha$ in the definition of PTSC.

We have analyzed similar settings in more detail in another paper~\cite{goel2020outside} assuming that every agent who submit a feedback is eligible for a refund. The following analysis in this paper is for a special case when only agents who submit $0$ as their feedback are eligible for a refund.

\paragraph{Beliefs and Belief Correlation.} The most important constituents of the peer-consistency framework are the agents' beliefs about the observations of their peers. We will let $P_i(x_{p} = x')$, for $x' \in \{0,1\}$, denote agent $i$'s (prior) belief about a randomly selected peer $p$'s observation $x_{p}$ on a question being $x'$. A standard assumption in the literature is that the priors are fully mixed, i.e $P_i(x_{p} = x') > 0, \forall x' \in \{0,1\}$. After the agent makes a private observation $x_{i}$ for a question, she updates her belief (posterior) about her peer's observation on that question only, to $P_i(x_{p} = x' | x_{i} = x)$. Given the beliefs of the agents, the following quantity will be useful:

$$	\small \beta = \min_{i} \Big(\frac{P_i(x_{p} = 1 | x_{i} = 1)}{P_i(x_{p} = 1)} - \frac{P_i(x_{p} = 0 | x_{i} = 1)}{P_i(x_{p} = 0)}\Big)
$$\smallskip

\noindent Intuitively, the quantity $\beta$ measures the correlation strength between the observations of agents. The assumption that $\beta > 0$ is standard in the literature of peer-consistency (e.g., see~\cite{jurca2005enforcing,witkowski2012peer}) and in fact, it is a prerequisite for the PTSC mechanism to guarantee truthful behavior. The assumption is rather obvious in binary answer settings: if an agent observes $1$, that can only increase her posterior belief about her stochastically-relevant peer also observing $1$. We will make the same assumption here, and we will use $\beta$ to quantify the scaling constant $\alpha$ that we need to use in PTSC, to overcome the lying incentives. \smallskip

\noindent We will also need the following quantity:
$$\gamma = \max_i P_i(x_p=0|x_i=1),$$
which measures the maximum over the posterior beliefs of any agent about her peer's observation being $0$, given that her own observation was $1$. 

\begin{thm}\label{thm:eqm}
	There is a value of the scaling constant $\alpha$ for which PTSC has a strict truthful equilibrium, even when agents have outside incentives. In particular, this is guaranteed when 
	$$\alpha > \frac{c \cdot \left(1 + (n-1)\gamma\right)}{n \cdot \beta}$$
\end{thm} 
\noindent Note that the scaling constant $\alpha$ decreases with increasing $n$.

\subsection{Making PTSC Profitable}

While we have shown in the above theorem that agents can be incentivized to be truthful with a large-enough choice of the scaling constant of PTSC,  the question that still remains is whether this is economically profitable. Are the rewards of peer-consistency so big that we end up paying more to agents, compared to what we would have paid as refund, if they all simply reported $y_i=0$? We answer this question below. First, we quantify the payments of PTSC.

\begin{thm}\label{thm:pts-payments}
	The total expected PTSC payment per agent that is enough to elicit truthful information in the presence of outside incentives is given by $\alpha$.
\end{thm} 
\noindent We note that this is not immediate from Theorem \ref{thm:eqm}, as the expected payment issued for each agent is $$\alpha \cdot \mathbb{E}\Big[\frac{\mathbb{1}_{y = y'}}{R_i(y)} - 1\Big]$$where expectation is taken with respect to the actual random variation in the true observations of the agents (for e.g. random variation in the reception of the service). The theorem follows from the fact that this quantity in expectation can be upper bounded by $1$.

Due to $\alpha$'s inverse dependence on $n$, the payments get smaller with increasing $n$. To provide some concrete intuition, we remark that the PTSC payment will generally only be a small fraction of the maximum outside incentive $\mathcal{R}$. For a concrete example, consider a web service that provides a good service $95\%$ of the time; this is a realistic assumption for many web-services that stay in business. Furthermore, let us assume a small correlation in the agents' observations such that $P(x_p = 1 | x_i = 1) = P(x_i = 1) + 0.01$, and similarly for $P(x_p = 0 | x_i = 0)$. In this case, the PTSC payments per agent are just $1.2\%$ of the refund payments $c$ even when there are just $10$ agents answering a question; the fraction quickly decreases to $0.7\%$ of $c$ when there are $25$ agents.

In many applications, the same entity is responsible for paying the refunds as well as the PTSC payment. For example, in the web service case, the service provider is responsible for collecting truthful data and also paying the refund based on the collected payment. For such cases, we measure the relative saving of this entity as: \[\text{relative saving:} \ \ \frac{nc - \mathcal{P}}{nc},\] where $\mathcal{P}$ is the total payment (PTSC payments + outcome dependent payment) under the scheme to all the agents.

One might feel inclined to believe that if agents were not strategic, we could hope for a saving of 100\%. However, this will be true only if the outcome is $0$, when the service is ``bad'' 0\% of the time, whereas in reality that may not be the case as we show in the next proposition. We will let $P(1)$ denote the probability that a report on the platform is $1$.

\begin{proposition}\label{prop:maximum-saving}
	If agents reported truthfully ignoring the possible refund payments, the platform could make an expected relative saving of up to $P(1)(2-P(1))$.
\end{proposition}

This is of course an ideal case; in the presence of rational agents, we quantify the saving of PTSC as follows.

\begin{thm}\label{thm:savings-eqm}
	The expected relative saving in payments made in the truth-telling equilibrium is at least
	$$P(1)(2-P(1)) - \frac{\alpha}{c},$$
	where $P(1)$ is the actual probability of a randomly selected report being $1$ in the truthful equilibrium.
\end{thm} 

\noindent Note that as long as the condition $P(1)(2-P(1)) > \alpha/c$ is satisfied, the lower bound on the saving is actually a positive number. Again, due to $\alpha$'s inverse dependence on $n$, the savings get bigger with increasing $n$ and are also always guaranteed to be positive given sufficient value of $n$.
\medskip

We refer the reader to~\cite{goel2020outside} for a discussion of the undesired equilibrium where all agents report $0$, and how it can be eliminated when there exists a strictly positive fraction of honest agents. Another simple way to eliminate this undesired equilibrium is to keep open the possibility of a trusted verification if the fraction of agents reporting $0$ exceeds a certain threshold.  If independent verification doesn't confirm poor service, agents who reported $0$ can suffer a penalty that outweighs the refunds by a large margin. A credible threat of such verification is enough to deter rational agents from playing the undesired equilibrium.

\section{Conclusions}
In this paper, we presented a novel system called Infochain that implements decentralized, trustelss and transparent oracles on the Ethereum blockchain. Contrary to earlier proposals on decentralized crowdsourcing systems, Infochain addresses the issue of truthfulness by implementing game-theoretic peer-consistency mechanisms. We prove that peer-consistency mechanisms can be be used to elicit truthful information even when the agents have outside incentives to misreport the information. For the first time, we discussed issues that arise in implementing these mechanisms in blockchain. The paper also presents an important new criterion for comparing or evaluating these mechanisms by their implementation complexity on the Ethereum blockchain.

\bibliographystyle{named}
\bibliography{ijcai20}

\end{document}